\documentclass{amsart} 
\usepackage{amsmath,amssymb,natbib,amsthm}

\title{Positivity and Transportation}

\newcommand{\BEAS}{\begin{eqnarray*}}
\newcommand{\EEAS}{\end{eqnarray*}}
\newcommand{\BEA}{\begin{eqnarray}}
\newcommand{\EEA}{\end{eqnarray}}
\newcommand{\BEQ}{\begin{equation}}
\newcommand{\EEQ}{\end{equation}}
\newcommand{\BIT}{\begin{itemize}}
\newcommand{\EIT}{\end{itemize}}
\newcommand{\BNUM}{\begin{enumerate}}
\newcommand{\ENUM}{\end{enumerate}}

\newcommand{\BA}{\begin{array}}
\newcommand{\EA}{\end{array}}
\newcommand{\BC}{\begin{center}}
\newcommand{\EC}{\end{center}}

\newcommand{\ones}{\mathbf 1}

\newcounter{exno}

{\begin{quote}}{\end{quote}}

\makeatother

\usepackage{dsfont}

\newcommand{\EORG}[1]{\boldsymbol{\textcolor{puorange}{#1}}}

\def\bk{\mathbf{k}}

\def\bU{\mathbf{U}}

\newcommand{\dotprod}[2]{\ensuremath{\langle #1 , #2\,\rangle}}

\usepackage{enumitem}
\usepackage{enumerate}
\usepackage{amsmath}

\newcommand{\smallmat}[1]{\left[\begin{smallmatrix}#1\end{smallmatrix}\right]}

\def\RR{\mathbb{R}}

\def\NN{\mathds{N}}

\def\Xcal{\mathcal{X}}

\def\Ncal{\mathcal{N}}

\def\OMIT#1{}

\DeclareMathOperator{\card}{card}

\DeclareMathOperator{\defi}{def}

\DeclareMathOperator{\defeq}{\overset{\defi}{=}}

\usepackage{enumerate,multirow}

\providecommand{\abs}[1]{\lvert#1\rvert}

\usepackage{tabularx}

\makeatletter
\newif\if@borderstar
\def\bordermatrix{\@ifnextchar*{%
  \@borderstartrue\@bordermatrix@i}{\@borderstarfalse\@bordermatrix@i*}%
}
\def\@bordermatrix@i*{\@ifnextchar[{%
  \@bordermatrix@ii}{\@bordermatrix@ii[()]}
}
\def\@bordermatrix@ii[#1]#2{%
  \begingroup
    \m@th\@tempdima8.75\p@\setbox\z@\vbox{%
      \def\cr{\crcr\noalign{\kern 2\p@\global\let\cr\endline }}%
      \ialign {$##$\hfil\kern 2\p@\kern\@tempdima & \thinspace %
      \hfil $##$\hfil && \quad\hfil $##$\hfil\crcr\omit\strut %
      \hfil\crcr\noalign{\kern -\baselineskip}#2\crcr\omit %
      \strut\cr}}%
    \setbox\tw@\vbox{\unvcopy\z@\global\setbox\@ne\lastbox}%
    \setbox\tw@\hbox{\unhbox\@ne\unskip\global\setbox\@ne\lastbox}%
    \setbox\tw@\hbox{%
      $\kern\wd\@ne\kern -\@tempdima\left\@firstoftwo#1%
        \if@borderstar\kern2pt\else\kern -\wd\@ne\fi%
      \global\setbox\@ne\vbox{\box\@ne\if@borderstar\else\kern 2\p@\fi}%
      \vcenter{\if@borderstar\else\kern -\ht\@ne\fi%
        \unvbox\z@\kern-\if@borderstar2\fi\baselineskip}%
        \if@borderstar\kern-2\@tempdima\kern2\p@\else\,\fi\right\@secondoftwo#1 $%
    }\null \;\vbox{\kern\ht\@ne\box\tw@}%
  \endgroup
}
\makeatother

\newtheorem{theorem}{Theorem}

\newtheorem{lemma}{Lemma}

\newtheorem{remark}{Remark}

\usepackage{graphicx}
\usepackage{xcolor}
\definecolor{darkblue}{rgb}{0,0,0.5} 
\usepackage{transparent}

\definecolor{orange}{rgb}{1,0.5,0}
\definecolor{puorange}{rgb}{0.70,0.15,0}
\def\UU{\mathbb{U}}

\def\NW{\mathbf{NW}}

\author{Marco Cuturi}
\address{Graduate School of Informatics, Kyoto University}
\email{mcuturi@i.kyoto-u.ac.jp}

\usepackage{wrapfig}

\begin{document}
\maketitle

\begin{abstract}	
We prove in this paper that the weighted volume -- or generating function -- of the set of integral transportation matrices between two integral histograms $r$ and $c$ of equal sum is a positive definite kernel of $r$ and $c$ when the set of considered weights forms a positive definite matrix. The computation of this quantity, despite being the subject of a significant research effort in algebraic statistics, remains an intractable challenge for histograms of even modest dimensions. We propose an alternative kernel which, rather than considering all matrices of the transportation polytope, only focuses on a sub-sample of its vertices known as its Northwestern corner solutions. The resulting kernel is positive definite and can be computed with a number of operations $O(R^2d)$ that grows linearly in the complexity of the dimension $d$, where $R^2$ --  the total amount of sampled vertices -- is a parameter that controls the complexity of the kernel.
\end{abstract}

\section{Introduction}
Suppose that among $30$ students in a classroom, $7$ and $23$ have light and dark colored eyes respectively. You are also told that $12$ of them have light hair while $18$ have dark hair. What are all the possible populations of the 4 subgroups of students with light/light, dark/dark, light/dark and dark/light eyes and hair color respectively? Such quantities can be arranged in a $2\times 2$ matrix whose row sum vector must be equal to $[7,23]^T$ and column sum vector must be equal to $[12,18]$, $\smallmat{3&4\\9&14}$ for instance, and more generally \emph{any} integer values in the dots below that satisfy these constraints:
$$\bordermatrix[{[]}]{%
& 12 & 18 \cr
7 & \bullet & \bullet\cr
23 & \bullet & \bullet \cr
}$$	
Alternatively, suppose that two bakeries in a small village produce daily $7$ and $23$ loafs of bread each, while two restaurants in the same area each need $12$ and $18$ loafs to serve their customers every day. What are all the possible morning delivery plans of bread loafs that the two bakeries and shops can agree upon? These seemingly trivial sets of matrices coincide, and are known  in the statistics and optimization literature as the sets of \emph{contingency tables} and \emph{transportation plans} respectively.

In statistics, the problem of enumerating all such tables arises naturally in hypothesis testing. Suppose that by entering the aforementioned classroom you observe that the actual repartition of these groups is $\smallmat{5&2\\7&16}$. Such an observation intuitively suggests that eye and hair color are related, but how confident should you be about this statement? In the $2\times 2$ case presented above, the Fisher exact test~\citep{yates1934contingency} answers that question by computing the probabilities of \emph{all} possible tables outcomes if one assumes that they have been generated as the product of independent Bernoulli variables with law $p_1=7/30$ and $p_2=12/30$. By comparing all these probabilities with that of the observed table, we can conclude how reliable an independence hypothesis would be. In optimization, given a $2\times 2$ cost matrix which describes the cost (in gas, calories or time) of bringing a loaf from each bakery to each shop, finding the delivery plan with minimal cost is known as a transportation problem. Transportation problems are an extremely general class of linear programs which are known to encompass all instances of network flows~\cite[p.274]{bertsimas1997introduction}.

Optimal transportation distances~\citep{rachev1998mass,villani09} are distances between probability densities which combine both perspectives outlined above, where the probabilistic view on contingency tables is matched with the goal of computing an optimal transportation plan between two marginal probabilities given a metric on the probability space of interest. Such distances have been widely used in computer vision following the impulsion of~\citet{rubner1997earth} who used it to compare histograms of image features. When used in information retrieval tasks, transportation distances fare usually better in practice than other classical distances for histograms~\citep{Pele-iccv2009}.

Transportation distances have however two notable drawbacks. First, from a geometric point of view, transportation distances are deficient in the sense that they are not negative definite nor Hilbertian. Negative definiteness carries many favorable properties, among which the possibility to create Euclidean embeddings from which the metric can be accurately recovered, as well as the possibility to turn the distance into a positive definite kernel by simple exponentiation, as a radial basis function. Because of this deficiency, there is no known positive definite counterpart to transportation distances that can leverage the complexity of the set of contingency tables. Second, from a computational point of view, the computational cost of computing transportation distances grows in most cases of interest at least quadratically in the dimension $d$ of the histograms, which can be prohibitive for many applications.

We try to address both issues in this work. The main contribution of this paper is theoretical: after providing some background material and motivation in Section~\ref{sec:back} we prove in Section~\ref{sec:trans} that the generating function of the set of all contingency tables between two integral histograms is a positive definite kernel. Our second contribution is practical: we propose in Section~\ref{sec:nwc} a positive definite kernel that leverages these ideas while still being computationally tractable. 

\section{Background}\label{sec:back}
\subsection{The Transportation Polytope and the Set of Contingency Tables} 
We review in this section a few definitions, notations and results of interest to prove our result. In the following, we write $\dotprod{\,\cdot\,}{\cdot}$ for both the Frobenius dot-product and the usual dot-product of vectors.

Given a dimension $d$ fixed throughout this paper, for two vectors $r,c\in \RR^d$, let $U(r,c)$ be the transportation polytope of $r$ and $c$, namely the subset of nonnegative matrices in $\RR^{d\times d}$ defined as:
$$U(r,c)\defeq \{X\in\RR_+^{d\times d}\; |\; X\ones_d=r, X^T\ones_d=c\},$$
where $\ones_d$ is the $d$ dimensional vector of ones. $U(r,c)$ contains all nonnegative $d\times d$ matrices with row and column sums $r$ and $c$ respectively. It is easy to check that $U(r,c)$ is non-empty if and only if all coordinates of $r$ and $c$ are non-negative and if the total masses of $r$ and $c$ are the same, that is $r^T\ones_d=c^T\ones_d$. We will consider in most of this work \emph{integral} vectors $r$ and $c$ taken in the set $\Sigma_N$ of $d$-dimensional integral histograms with total mass $N\in\NN$,
$$
\Sigma_d^N \defeq \{r \in\NN^{d} \;|\; r_1+\cdots+r_d = N\}.
$$
We will also focus accordingly on the subset $\UU(r,c)$ of $U(r,c)$ that contains all integral transportation matrices, alternatively known as \emph{contingency tables}~\citep{lauritzen1982lectures,everitt1992analysis}:
$$\UU(r,c)\defeq U(r,c) \cap \NN^{d\times d}.$$

\subsection{Weighted Volumes of Contingency Tables and Particular Cases of Positivity} 
Ranging from early work by~\citet{yates1934contingency,good1976} to~\citet{diaconisefron,cryan2003polynomial,chen2005sequential}, the computation of elementary statistics about $\UU(r,c)$ has attracted considerable attention. Many of the ideas of this paper build upon recent work by~\citeauthor{barvinok2008enumerating}, most notably on his study of the generating function of $\UU(r,c)$, defined for $M\in \RR^{d\times d}$ as
$$V(r,c\,;M)\defeq \sum_{X\in \UU(r,c)} e^{-\dotprod{X}{M}}.$$
The generating function can be related to the \emph{weighted} volume~\citep[p.2]{barvinok2008enumerating} of $\UU(r,c)$, defined for any nonnegative $d\times d$ matrix $K\in\RR_+^{d\times d}$ as:
$$T(r,c\,;K) \defeq \sum_{X\in \UU(r,c)} \prod_{ij}^d k_{ij}^{x_{ij}}.$$ 
Both definitions are equivalent since if we agree that $k_{ij}=e^{-m_{ij}}$ then $T(r,c\,;K)=V(r,c\,;M)$. Because all of our results rely on $K$'s properties, we will mostly use the weighted volume formulation in this paper. Some sections in this paper, notably \S\ref{subsec:rel} below and \S\ref{sec:nwc}, are better understood with the generating function formulation.

~\citet[Prop.2]{cuturi07permanents} proved that the cardinal of the set $\UU(r,c)$ is a positive definite kernel of $r$ and $c$ using the Robinson-Schensted-Knuth bijection~\citep{Knuth70} that maps each contigency table to a pair of Young tableaux with contents $r$ and $c$ and the same pattern. It is easy to see that the cardinal of $\UU(r,c)$ is equal to $T(r,c;\ones_{d\times d})$ or $V(r,c\,;\mathbf{0}_{d\times d})$. ~\citet[Prop.1]{cuturi07permanents} also proved that $T(r,c\,;K)$ is a positive definite kernel of $r$ and $c$ if both are \emph{binary} histograms and $K$ is a nonnegative $d\times d$ positive definite matrix. Since the computation of $T$ entails in that case the computation of the permanent of a Gram matrix,~\citet{cuturi07permanents} called this kernel the permanent kernel. The main contribution of our paper is to prove in Theorem~\ref{theo:genfpsd} that the map $(r,c)\in\Sigma_d^N \mapsto T(r,c\,;K)$ is positive definite whenever $K$ is a $d\times d$ positive definite matrix.

\begin{figure}
\BC\scalebox{.75}{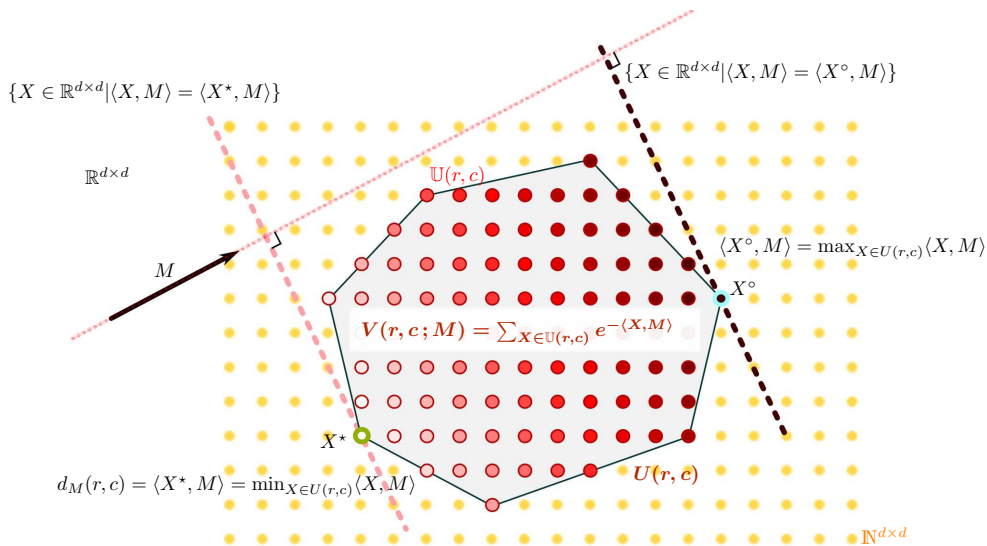}\EC
\caption{Schematic representation of the set $\UU(r,c)$ of contingency tables seen as the intersection between the lattice of integral matrices $\NN^{d\times d}$ with the transportation polytope $U(r,c)$. Each red dot stands for an integral plan $X\in\UU(r,c)$. The inner color in each red dot stands for the value of $\dotprod{X}{M}$, which can be seen to go gradually from $\dotprod{X^\star}{M}$ to $\dotprod{X^\circ}{M}$, that is from the minimum to the maximum of $\dotprod{\cdot}{M}$ over $U(r,c)$, or equivalently $\UU(r,c)$. The generating function $V(r,c;M)$ of $\UU(r,c)$ considers the contributions of \emph{all} contingency tables.}\label{fig:mainfig}
\end{figure}

\subsection{Relationships with the Optimal Transportation Distance}\label{subsec:rel}
Given a $d\times d$ cost matrix $M$, one can quantify the cost of mapping $r$ to $c$ using a transportation matrix $X$ as $\dotprod{X}{M}$.
The minimum of this cost is called the optimal transportation cost, defined as:
$$
d_M(r,c) \defeq \min_{X\in U(r,c)} \dotprod{X}{M}.
$$
A classical result of optimization in network flows~\citep[Theo. 7.5]{bertsimas1997introduction} guarantees the existence of a contingency table $X^\star\in\UU(r,c)$ which achieves this minimum,  as schematically represented in Figure~\ref{fig:mainfig}. Such an optimal table $X^\star$ can be obtained algorithmically in polynomial time~\cite[\S9]{ahuja1993network}.

The minimal cost $d_M(r,c)$ turns out to be a distance~\cite[\S6.1]{villani09} whenever the matrix $M$ is itself a metric. This distance is also known as the Wasserstein distance, Monge-Kantorovich's, Mallow's or Earth Mover's~\citep{rubner1997earth} in the computer vision literature. The transportation distance is not negative definite in the general case, as shown by counterexamples~\citep{naor-2005} and embedding distortion results~\citep{indyk2009}. Although some metrics $M$ can yield a negative definite distance\footnote{Setting $M=\ones_{d\times d}-I_d$ yields the total variation distance between discrete probabilities, which is half the Manhattan or $l_1$ distance between $r$ and $c$. All these distances are known to be negative definite.}, characterizing the negative definiteness of $d_M$ remains an open question. Despite this fact, transportation distances have been used in practice to derive a \emph{pseudo}-positive definite kernel: both~\citet[\S4.C]{emd2004} or~\citet[\S2.3]{emd2006} introduce the exponential of (minus) the minimum of $\dotprod{X}{M}$, 
\begin{equation}\label{eq:km}k_M (r,c) = e^{-d_M(r,c)} = \exp\left(-\min_{X\in U(r,c)}\dotprod{X}{M}\right),\end{equation}
to form an undefinite kernel which can be used to compare histograms in practice. We prove that, although the value $\exp(- \langle X^\star, M\rangle)$ in itself is not a positive definite kernel, the sum of each term $\exp(- \langle X, M\rangle)$ over \emph{all} possible contingency tables in $\UU(r,c)$ is positive definite when $M$ has suitable properties. The generating function $V_{rc}$ can be interpreted as the exponential of (minus) the soft-minimum of $\dotprod{X}{M}$ over all contingency tables,
$$
V(r,c\,;M) = \exp\left(-\,\underset{X\in \UU(r,c)}{\text{softmin}}\,\dotprod{X}{M}\right)= e^{\log\sum_{X\in \UU(r,c)} e^{-\dotprod{X}{M}}}= \sum_{X\in \UU(r,c)} e^{-\dotprod{X}{M}},
$$
where the soft-minimum of a finite family of scalars $(u_i)$ is $$\,\underset{i}{\text{softmin}}\,u_i\,\defeq -\log\sum_{i} e^{-u_i}.$$ This expression relates our results in this work to previous applications of soft minimums to derive positive definite kernels from combinatorial distances for strings~\citep{VerSaiAku04}, time series~\citep{cuturi07kernelSHORT} and trees~\citep{shin2011mapping}. These ideas are summarized in Figure~\ref{fig:mainfig}.

\subsection{Generalized Permutations}\label{subsec:genperm}
We close this section by providing some tools to prove the result. We write $S_N$ for the group of permutations over the set $\{1,\cdots,N\}$. For any vector $\alpha$ of size $N$ and permutation $\pi\in S_N$, we write $\alpha_\pi$ for the permuted vector with coordinates $\alpha_\pi=[\alpha_{\pi(1)}\,\alpha_{\pi(2)}\,\cdots\,\alpha_{\pi(N)}]$ and $\alpha_{p\cdot\cdot q}$ for the subvector $[\alpha_{p}\,\cdots\,\alpha_{q}]$ when $1\leq p \leq q \leq N$. For two vectors $\rho,\gamma$ of $\{1,\cdots,d\}^N$, the $2\times N$ array
$$
(\rho\,;\gamma) \defeq \begin{bmatrix} \rho_1 & \rho_2 &\cdots & \rho_N \\ \gamma_1 & \gamma_2 &\cdots & \gamma_N \\ \end{bmatrix},
$$
is called a generalized permutation~\citep{Knuth70}. To any generalized permutation $(\rho\,;\gamma)$ corresponds a $d\times d$ integral matrix $\chi(\rho\,;\gamma)$ defined as~\cite[p.41]{fulton1997young}: 
\begin{equation}\label{eq:fulton}[\chi(\rho\,;\gamma)]_{ij} \defeq \sum_{n=1}^N\ones_{\rho_t=i}\cdot\ones_{\gamma_t=j},\quad 1\leq i,j\leq d.\end{equation}
Consider the following example where $d=3,N=8$ and $$\rho=\begin{bmatrix}1\,2\,2\,2\,1\,3\,1\,3\;\end{bmatrix}, \gamma=\begin{bmatrix}1\,1\,2\,1\,3\,3\,3\,3\;\end{bmatrix}, 
(\rho\,;\gamma) = \begin{bmatrix}1\,2\,2\,2\,1\,3\,1\,3 \\1\,1\,2\,1\,3\,3\,3\,3\end{bmatrix}, \chi(\rho\,;\gamma) =\begin{bmatrix} 1 & 0 & 2\\ 2&1 &0 \\ 0 & 0 & 2 \end{bmatrix}.
$$
If we consider now the permutation $\pi=[3\,6\,8\,5\,2\,1\,4\,7]$ we have that
$$\rho=\begin{bmatrix}1\,2\,2\,2\,1\,3\,1\,3\;\end{bmatrix}, \gamma_\pi=\begin{bmatrix}2\,3\,3\,3\,1\,1\,1\,3\;\end{bmatrix}, 
(\rho\,;\gamma_\pi) = \begin{bmatrix}1\,2\,2\,2\,1\,3\,1\,3 \\2\,3\,3\,3\,1\,1\,1\,3\end{bmatrix}, \chi(\rho\,;\gamma_\pi) =\begin{bmatrix} 2 & 1 & 0\\ 0&0 &3 \\ 1 & 0 & 1 \end{bmatrix}.
$$
Note that if $\rho$ and $\gamma$ have respectively $r_i$ and $c_i$ elements $i$ among their $N$ coefficients for all $1\leq i\leq d$, then $\chi(\rho\,;\gamma)\in \UU(r,c)$. One can see above that the corresponding histograms are $r=[3,3,2]$ and $c=[3,1,4]$ and that both $\chi(\rho\,;\gamma)$ and $\chi(\rho\,;\gamma_\pi)$ have row and column sums $r$ and $c$.

\section{The Weighted Volume as a Positive Definite Kernel}\label{sec:trans}

\begin{theorem}\label{theo:genfpsd}
Let $K\in\RR_{+}^{d\times d}$. The map $(r,c)\mapsto T(r,c\,;K)$ is positive definite if $K$ is positive definite.
\end{theorem}

The proof relies on the following observation:~\citet{barvinok2008enumerating} showed that the weighted volume of $\bU(r,c)$ of two integral histograms $r$ and $c$ of total mass $N$ can be formulated as the expectation of the permanent of a random $N\times N$ matrix. To do so,~\citeauthor{barvinok2008enumerating} shows that the weighted volume -- a sum indexed over all \emph{contigency tables} $X\in\UU(r,c)$, can be rewritten as a sum indexed over all \emph{permutations} $\pi$ in $S_N$, up to a correcting term known as the Fisher-Yates statistic (Equation~\eqref{eq:fy} in the Appendix). The crux of~\citeauthor{barvinok2008enumerating}'s proof lies in a randomization scheme -- using draws from the exponential law -- to cancel out the Fisher-Yates statistic. We adopt a similar route to prove the positivity of $T$, by proving that the inverse of the Fisher-Yates statistic -- defined as $\bk_2$ below -- is itself positive definite to obtain the result.

\begin{proof} Suppose that $K\in\RR_+^{d\times d}$ is positive definite and consider two integral histograms $r,c$ in $\Sigma_d^N$. We represent $r$ as a $N$-dimensional vector $\rho\in\{1,\cdots,d\}^N$,
$$\rho\defeq [\,\underbrace{1,\cdots,1}_{r_1 \text{ times }},\underbrace{2,\cdots,2}_{r_2 \text{ times }},\cdots,\underbrace{d,\cdots,d}_{r_d \text{ times }}\,],$$
and consider the analogous representation $\gamma$ for $c$. Let $\mathbf{k}_1$ and $\mathbf{k}_2$ be the following kernels on $(\rho,\gamma)$:
$$
\begin{aligned}
	\mathbf{k}_1(\rho,\gamma)&= \prod_{t=1}^N k(\rho_t,\gamma_t)\;, \text{ where } k(i,j) =k_{ij} \text{ for } 1\leq i,j\leq d,\\
	\mathbf{k}_2(\rho,\gamma)&= \frac{1}{r_1!\cdots r_d!} \cdot \frac{1}{c_1!\cdots c_d!} \prod_{ij}^d x_{ij}!\;, \text{ where } X=\chi(\rho\,;\gamma). \quad (\text{see \S\ref{subsec:genperm}, Eq.~\eqref{eq:fulton}})
\end{aligned}
$$
The kernel $\mathbf{k}_2$ is the inverse of the Fisher-Yates statistic  (Equation~\eqref{eq:fy} in the Appendix) associated to an integral transportation table $X$ and its marginals $r$ and $c$. $\mathbf{k}_1$ is trivially positive definite. The first group of terms of $\mathbf{k}_2$ is trivially positive definite  as a product $f(r)f(c)$ where $f(r)=\frac{1}{r_1!\cdots r_d!}$. We prove that the other term, the product of factorials of $x_{ij}$, is positive definite in Lemma~\ref{lem:fact} using the proof strategy of a related result provided in Lemma~\ref{lem:fac}. Lemma~\ref{lem:perm} proves that when a kernel $\kappa$ on two vectors is symmetric (the definition is provided in the lemma), the sum $\sum_{\pi\in S_N}\kappa(\rho,\gamma_\pi)$ is itself positive definite. We use this result on the product $\kappa(\rho,\gamma)=\mathbf{k}_1(\rho,\gamma) \,\mathbf{k}_2(\rho,\gamma)$ which is trivially symmetric as the product of two symmetric kernels. We then prove in Lemma~\ref{lem:decomp} that
$$
\sum_{\pi\in S_N} \kappa(\rho,\gamma_\pi) = T(r,c\,;K).
$$
Since the summation over all permutations in the left hand side is positive definite by Lemma~\ref{lem:perm}, we conclude that $T(r,c\,;K)$ is itself a positive definite kernel as the product of two positive definite kernels.
\end{proof}
\section{Northwestern Kernel}\label{sec:nwc}
The weighted volume $T(r,c\,;K)$ cannot be computed exactly even for small dimensions $d$, and approximations~\citep{barvinok2008enumerating} are currently both too expensive and too loose to be of practical interest in a machine learning context. We adopt in this section an alternative approach, in which we propose to restrict the sum of elementary contributions $\exp(-\dotprod{X}{M})$ to a subset of extreme points of $U(r,c)$ and obtain a kernel whose computational complexity grows linearly in both the dimension $d$ and the size of the sample of extreme points. The main tool for this approach is provided by the Northwestern corner rule to generate a vertex of $U(r,c)$, which we recall in Section~\ref{subsec:nwc}. We define  the Northwester kernel in Section~\ref{subsec:sam} and prove that it is positive definite. For any matrix $M\in\RR^{d\times d}$, we write $M_{\sigma\sigma'}$ for the row and column permuted matrix whose $i,j$ element is $m_{\sigma(i)\sigma'(j)}$.
\subsection{The Northwestern Corner Rule to Generate Vertices of $U(r,c)$}\label{subsec:nwc}
The Northwestern corner rule is a heuristic that produces a vertex of the polytope $U(r,c)$ in up to $2d$ operations. The rule starts by giving the highest possible value to $x_{11}$, and at each step when a highest possible value is given to entry $x_{ij}$ it moves on to $x_{ij+1}$ in case $x_{ij}$ filled column $j$, or $x_{i+1j}$ in case $x_{ij}$ filled row $i$. The rule proceeds until $x_{nn}$ has received a value.
Here is an example of this sequence assuming $r=[2,5,3]$  and $c=[5,1,4]$:
$$\begin{bmatrix} \bullet & 0 & 0 \\ 0 & 0 & 0 \\ 0& 0 & 0\end{bmatrix} \rightarrow \begin{bmatrix} 2 & 0 & 0 \\ \bullet & 0 & 0 \\ 0& 0 & 0\end{bmatrix} \rightarrow \begin{bmatrix} 2 & 0 & 0 \\ 3 & \bullet & 0 \\ 0& 0 & 0\end{bmatrix} \rightarrow \begin{bmatrix} 2 & 0 & 0 \\ 3 &1 &\bullet \\ 0& 0 & 0\end{bmatrix} \rightarrow \begin{bmatrix} 2 & 0 & 0 \\ 3 &1 &1 \\ 0& 0 & \bullet\end{bmatrix} \rightarrow \begin{bmatrix} 2 & 0 & 0 \\ 3 &1 &1 \\ 0& 0 & 3\end{bmatrix}$$
We write $\NW(r,c)$ for the unique Northwestern corner solution that can be obtained through this heuristic. There is, however, a much larger number of Northwestern corner solutions that can be obtained by permuting arbitrarily the order of $r$ and $c$ separately, computing the corresponding Northwestern corner table, and recovering a table of $\UU(r,c)$ by inverting again the order of columns and rows. Setting $\sigma=(3,1,2),\sigma'=(3,2,1)$ we have that $r_\sigma=[3,2,5], c_{\sigma'}=[4,1,5]$ and $\sigma^{-1}=(2,3,1),\sigma'=(3,2,1)$. Observe that:
$$
\NW(r_\sigma,c_\sigma') = \begin{bmatrix} 3 & 0 & 0 \\ 1 & 1 & 0 \\ 0& 0 & 5\end{bmatrix} \in \UU(r_\sigma,c_{\sigma'}),\;\NW_{\sigma^{-1}\sigma'^{-1}}(r_\sigma,c_{\sigma'})= \begin{bmatrix} 0 & 1 & 1 \\ 5 & 0 & 0 \\ 0& 0 & 3\end{bmatrix}\in \UU(r,c).
$$
Let $\Ncal(r,c)$ be the set of all Northwestern corner solutions that can be produced this way:
$$\Ncal(r,c)\defeq\{ \NW_{\sigma^{-1}\sigma'^{-1}}(r_\sigma,c_{\sigma'}), \sigma,\sigma'\in S_d\}.$$
Note that all Northwestern corner solutions only have by construction up to $2d-1$ nonzero elements. The Northwestern corner rule produces a table which is by construction unique for $r$ and $c$, but there is an exponential number of pairs or row/column permutations $(\sigma,\sigma')$ that may share the same table~\citep[p.2]{stougie2002polynomial}. $\Ncal(r,c)$ is a subset of the set of extreme points of $U(r,c)$~\citep[Corollary 8.1.4]{brualdi2006combinatorial}. $\NW(r,c)$ is an optimal transportation between $r$ and $c$ if the cost matrix $M$ is a Monge matrix~\citep{hoffman1961simple}, that is a matrix $M$ that satisfies the inequalities $$\forall 1 \leq i,j,k,l\leq d, \quad m_{ij}+m_{kl}\leq m_{il}+m_{kj}.$$ Note however that a distance matrix cannot be a Monge matrix since the inequality above applied to $k=j$ and $l=i$ would imply that $0<2m_{ij}\leq m_{ii}+m_{jj}=0$.

\subsection{Random Sampling of Northwestern Corner Solutions}\label{subsec:sam}
We propose in this section a kernel which uses arbitrary row/column permutations of $r$ and $c$ to recover extreme points of $\UU(r,c)$ and sum their individual contribution:
\begin{theorem}Let $R$ be an arbitrary subset of permutations in $S_d$. The Northwestern kernel sampled on $R$ and parameterized by a matrix $M$, defined as
	$$
	N(r,c\,;K,R) \defeq \sum_{\sigma,\sigma'\in R} \exp\left(-\dotprod{M}{\NW_{\sigma^{-1}\sigma'^{-1}}(r_\sigma,c_{\sigma'})}\right),
	$$
	is a positive definite kernel if $K$, the element-wise exponential of $-M$, is positive definite.
\end{theorem}
\begin{proof}
	As in the proof of Theorem~\ref{theo:genfpsd}, consider the representation of an integral histogram $r\in\Sigma_d^N$ as a $N$ dimensional vector $\rho$ that replicates $r_i$ times the index $i$ for all $i$ from $1$ to $d$. We also define, for any permutation $\sigma$ of $S_d$, the vector $\rho_\sigma$ as
	$$\rho_\sigma \defeq [\,\underbrace{\sigma(1),\cdots,\sigma(1)}_{r_{\sigma(1)} \text{ times }},\underbrace{\sigma(2),\cdots,\sigma(2)}_{r_{\sigma(2)} \text{ times }},\cdots,\underbrace{\sigma(d),\cdots,\sigma(d)}_{r_{\sigma(d)} \text{ times }}\,].$$
$\rho_\sigma$ for $\sigma\in S_d$ should not be confused with $\rho_\pi$ for $\pi\in S_N$ (\S\ref{subsec:genperm}): for any permutation $\sigma\in S_d$ there exists at least one permutation $\pi\in S_N$ such that $\rho_\sigma=\rho_\pi$ but the converse is not usually true. We show in Lemma~\ref{lem:nwc} that for $\sigma,\sigma'\in S_d$, $\NW_{\sigma^{-1}\sigma'^{-1}}(r_\sigma,c_{\sigma'})=\chi(\rho_\sigma,\gamma_{\sigma'})$, and thus,
$$N(r,c\,;K,R) = \sum_{\sigma,\sigma'\in R} e^{-\dotprod{M}{\chi(\rho_\sigma,\gamma_{\sigma'})}} = \sum_{\sigma,\sigma'\in R} \mathbf{k_1}(\rho_\sigma,\gamma_{\sigma'}),$$
where $\mathbf{k_1}$ is defined in Theorem~\ref{theo:genfpsd}. $N(r,c\,;K,R)$ is positive definite as a convolution kernel.
\end{proof}

\begin{lemma}\label{lem:nwc}
Let $\sigma$ and $\sigma'$ be two permutations of $S_d$. Then $$\NW_{\sigma^{-1}\sigma'^{-1}}(r_\sigma,c_{\sigma'})=\chi(\rho_\sigma,\gamma_{\sigma'}).$$
\end{lemma}
\begin{proof}We write $E_{ij}$ for the $d\times d$ matrix of zeros except for the $(i,j)$ element set to $1$. We prove the result by induction on the total mass $N$. For $N=1$ the result is trivial since the only transportation matrix in $U(r,c)$ in that case is $E_{\sigma(i_1)\sigma(i_2)}$, where $i_1$ and $i_2$ are such that $r_{i_1}=c_{i_2}=1$. Suppose now that the result is true for all histograms of mass $N$ and consider the case where $r^T\ones_d=c^T\ones_d=N+1$. Let $i_1$ and $i_2$ be the smallest indices such that $r_{\sigma(i)}>0$ and $c_{\sigma'(i)}>0$ respectively. As a consequence, the first elements of $\rho_\sigma$ and $\gamma_{\sigma'}$ are $\sigma(i_1)$ and $\sigma(i_2)$ respectively. Consider the two vectors $\rho_*$ and $\gamma_*$ of length $N$ equal to $\rho_\sigma$ and $\gamma_{\sigma'}$ \emph{without} these two first elements. Setting $\tilde{r}$ and $\tilde{c}$ to $r$ and $c$ except for the fact that $\tilde{r}_{\sigma(i_1)}=r_{\sigma(i_1)}-1$ and $\tilde{c}_{\sigma(i_2)}=r_{\sigma(i_2)}-1$, we have by induction that $\NW_{\sigma^{-1}\sigma'^{-1}}(\tilde{r}_\sigma,\tilde{c}_{\sigma'})=\chi(\rho_*,\gamma_*),$ since the two histograms have total mass $N$ and their representations are respectively $\rho_*$ and $\gamma_*$. By definition of the Northwestern corner rule, adding a unit of mass to the $i_1$'s and $i_2$'s components of $\tilde{r}_\sigma$ and $\tilde{c}_{\sigma'}$ only changes the very first iteration of the rule, since all coordinates of $\tilde{r}_\sigma$ and $\tilde{c}_{\sigma'}$ up to but not including $i_1$ and $i_2$  respectively are null by construction. Applying the rule yields a transportation table with an added unit in location $(i_1,i_2)$, providing thus the identity
	$$\NW(r_\sigma,c_{\sigma'}) = \NW(\tilde{r}_\sigma,\tilde{c}_{\sigma'}) + E_{i_1i_2},$$ which implies that 
\begin{equation}\label{eq:nw}\NW_{\sigma^{-1}\sigma'^{-1}}(r_\sigma,c_{\sigma'})= \NW_{\sigma^{-1}\sigma'^{-1}}(\tilde{r}_\sigma,\tilde{c}_{\sigma'}) + E_{\sigma(i_1)\sigma'(i_2)}.
\end{equation}
By definition of $\chi$ we have that
\begin{equation}\label{eq:chi}
\chi(\rho_\sigma\gamma_\sigma)= \chi(\rho_*,\gamma_*) + E_{\sigma(i_1)\sigma'(i_2)}
\end{equation}
we get by combining Equations~\eqref{eq:chi} and~\eqref{eq:nw} above with the induction hypothesis that $\NW_{\sigma^{-1}\sigma'^{-1}}(r_\sigma,c_{\sigma'})=\chi(\rho_\sigma,\gamma_{\sigma'})$.
\end{proof}

\begin{remark}The evaluation of $N(r,c\,;K,R)$ requires $O(d\abs{R}^2)$ steps since computing each of the $\abs{R}^2$ contributions $\exp(-\dotprod{M}{\NW_{\sigma^{-1}\sigma'^{-1}}(r_\sigma,c_{\sigma'})})$ for a couple $\sigma,\sigma'$ requires up to $2d$ products. The size of $R\subset S_d$ can be controlled from a few permutations to an exhaustive enumeration, which would entail an overall complexity of the order of $O(dd!^2)$.\end{remark}
	
\section{Conclusion and Future Work}
We have proved in this paper that the fundamental ingredient of transportation distances, the polytope of contingency tables, can be used to define a positive definite kernel between two histograms. While the cost matrix of a transportation problem between two histograms $r$ and $c$ needs to be a distance matrix for the optimum to be itself a distance of $r$ and $c$, we have proved that the generating function of the same polytope is positive definite whenever the cost matrix is itself positive definite. This quantity is computationally intractable, and we have resorted to a summation that only considers a subset of extreme points of the polytope to define the north-western kernel. Future research includes the proposal of suitable subsets $R$ of permutations of $S_d$ tuned with data, as well as other approximation schemes.

\section*{Appendix: Intermediate Results for the Proof of Theorem~\ref{theo:genfpsd}}

\begin{lemma}\label{lem:fac}Let $a,b\in\{0,1\}^N$ be two binary vectors. The kernel $(a,b)\mapsto \dotprod{a}{b}!$ is positive definite.
\end{lemma}
\begin{proof} For $N=1$ the kernel is always equal to $1$ and is thus trivially positive definite. For $N>1$, the recursion $\dotprod{a}{b}!=\dotprod{a_1^{N-1}}{b_1^{N-1}}!\,(a_{N}b_{N}\dotprod{a_1^{N-1}}{b_1^{N-1}}+1)$ provides the expression
$$\dotprod{a}{b}! = \prod_{t=1}^{N-1} \left(a_{t+1}b_{t+1}\dotprod{a_{1\cdot\cdot t}}{b_{1\cdot\cdot t}}+1\right),$$ 
which shows that $\dotprod{a}{b}!$ is the product of $N-1$ positive definite kernels on different features of $a$ and $b$.\end{proof}
\begin{remark}Rather than the lemma itself, we will use the identity above in the proof of Lemma~\ref{lem:fact}. We conjecture that this result can be extended to integral vectors. Numerical counterexamples show that this result cannot be generalized to vectors of $\RR^N$ through Euler's or Hadamard's $\Gamma$ function.\end{remark}

\begin{lemma}\label{lem:fact}Let $\rho,\gamma\in\{1,\cdots,d\}^N$. The kernel $(\rho,\gamma)\mapsto \prod_{ij} x_{ij}!$, where $X=\chi(\rho;\gamma)$, is positive definite.
\end{lemma}
\begin{proof} An integral vector $\rho \in \{1,\cdots,d\}^N$ with $N$ components can be represented as a family of $d$ binary row vectors $\rho^1,\cdots,\rho^d$ of length $N$ where for $n\leq N$, $\rho^i_n\defeq \ones_{\rho_n=i}$. For instance,

$$\text{if }\rho=\begin{bmatrix}1\,1\,2\,2\,2\,1\,3\,1\,3\,3\end{bmatrix},\text{ then }  \begin{bmatrix}\rho^1\\\rho^2\\\rho^3\end{bmatrix}=\begin{bmatrix}
1&1&0&0&0&1&0&1&0&0\\
0&0&1&1&1&0&0&0&0&0\\
0&0&0&0&0&0&1&0&1&1\\
\end{bmatrix}$$
These $d$ binary vector representations can be used to obtain the matrix $\chi(\rho\,;\gamma)$. Indeed, it is easy to check that if $X=\chi(\rho\,,\gamma)$ then $x_{ij}=\dotprod{\rho^i}{\gamma^j}$. As a consequence, we have that for all indices $i,j$ the coefficient $x_{ij}!=\dotprod{\rho^i}{\gamma^j}!$. We obtain that the product of factorials
$$
\prod_{ij}^d x_{ij}! = \prod_{i,j}^d\dotprod{\rho^i}{\gamma^j}!,
$$
is thus a product of kernels evaluated on all possible pairs among the $d\times d$ representations for $\rho$ and $\gamma$.  Although one might be tempted to interpret this product as a convolution kernel~\citep{haussler99convolution} or a mapping kernel~\citep{shin2008generalization}, one should recall that such results only apply to \emph{sums} of local kernels and not to \emph{products}. Such products of kernels on parts are not, as simple counterexamples can show, positive definite in the general case. Using the decomposition which was used in the proof of Lemma~\ref{lem:fac}, we have however that:
$$
\begin{aligned}\prod_{ij}^d x_{ij}! &= \prod_{i,j}^d\dotprod{\rho^i}{\gamma^j}! = \prod_{i,j}^d \prod_{t=1}^{N-1} \left(\rho^i_{t+1}\gamma^j_{t+1}\dotprod{\rho^i_{1\cdot\cdot t}}{\gamma^i_{1\cdot\cdot t}}+1\right),\\
&= \prod_{t=1}^{N-1} \prod_{i,j}^d \left(\rho^i_{t+1}\gamma^j_{t+1}\dotprod{\rho^i_{1\cdot\cdot t}}{\gamma^j_{1\cdot\cdot t}}+1\right) = \prod_{t=1}^{N-1} \left(1+\sum_{i,j}^d \rho^i_{t+1}\gamma^j_{t+1}\dotprod{\rho^i_{1\cdot\cdot t}}{\gamma^j_{1\cdot\cdot t}}\right),
\end{aligned}
$$
where we have used in the last operation the fact that only one of all $d^2$ products $(\rho^i_{t+1}\gamma^j_{t+1})_{ij}$ is nonzero, since
$$
\rho^i_{t+1}\gamma^j_{t+1}=\begin{cases} 1, \text{ if } \rho_{t+1}=i \text{ and } \gamma_{t+1}=j, \\ 0,  \text {else.}\end{cases}
$$
The product of factorials is thus a product of $N-1$ positive definite kernels indexed by $t$ and defined on $\rho$ and $\gamma$, where each of these $N-1$ kernel is $1$ plus a convolution kernel operating on the $d$ decompositions of $\rho_{1\cdot\cdot t}$ and $\gamma_{1\cdot\cdot t}$ as $d$ binary feature vectors, that is
$$
\prod_{ij}^d x_{ij}! = \prod_{t=1}^{N-1} \left(1+k_t(\rho,\gamma)\right); 
$$
where
$$k_{t}(\rho,\gamma)=\sum_{i,j}^d h_t(\rho^i,\gamma^j) \text{ and } h_t(a,b) = a_{t+1}b_{t+1}\dotprod{a_{1\cdot\cdot t}}{b_{1\cdot\cdot t}}.$$
\end{proof}

\begin{lemma}\label{lem:perm} Let $\alpha=(\alpha_1,\cdots,\alpha_N)$ and $\beta=(\beta_1,\cdots,\beta_N)$ be two lists of $N$ elements in a set $\Xcal$. Let $k$ be a symmetric kernel in $\Xcal^N$, that is a kernel invariant under a permutation of the order of both $\alpha$ and $\beta$: $\forall \pi\in S_N,\; k(\alpha,\beta)=k(\alpha_\pi,\beta_\pi).$
Then $(\alpha,\beta)\mapsto \sum_{\pi\in S_N} k(\alpha,\beta_{\pi})$ is positive definite.
\end{lemma}
\begin{proof}
The function $g$ defined below is, by~\citeauthor{haussler99convolution}'s (\citeyear{haussler99convolution}) convolution kernels framework, a positive definite kernel of $\alpha$ and $\beta$:
$$
g(\alpha,\beta)=\sum_{\pi'\in S_N} \sum_{\pi\in S_N} k(\alpha_{\pi'},\beta_{\pi}).
$$
Using the symmetric property of $\kappa$, we have that
$$
g(\alpha,\beta)=\sum_{\pi'\in S_N} \sum_{\pi\in S_N} k(\alpha,\beta_{{\pi'}^{-1}\circ\pi}) = N!\sum_{\pi\in S_N} k(\alpha,\beta_{\pi}).
$$
which proves the result.
\end{proof}

\begin{lemma}\label{lem:decomp}
	$\sum_{\pi\in S_N} \kappa(\rho,\gamma_\pi) = r_1!\cdots r_d! \cdot c_1!\cdots c_d! \,T(r,c\,;K)$
\end{lemma}
\begin{proof}For any couple of vectors $\rho,\gamma$ we have that both $\mathbf{k}_1$ and $\mathbf{k}_2$ only depend on $X=\chi(\rho\;;\gamma)$. This is implicitly the case in the definition of $\mathbf{k}_2$ and one can check that
$$
\mathbf{k}_1(\rho,\gamma)= \prod_{t=1}^N k(\rho_t,\gamma_t) = \prod_{ij}^d k_{ij}^{x_{ij}}, \text{ where } X=\chi(\rho\;;\gamma).
$$
With every permutation $\pi$ of we associate a transportation table $\chi(\rho\,;\gamma_\pi)$ which we call the pattern of $\pi$. Following~\citep[\S2,p.7]{barvinok2008enumerating}, we know that the number of permutations $\pi$ that share the same pattern $X$ for $X\in \UU(r,c)$ only depends on $X$, $r$ and $c$ through a formula known as the Fisher-Yates statistic $n(X)$ of $X$,
\begin{equation}\label{eq:fy}
n(X)\defeq \card\{\pi\in S_N | \,\chi(\rho\,;\gamma_\pi) = X\} = \frac{r_1!\cdots r_d! \cdot c_1!\cdots c_d!}{\prod_{ij}x_{ij}!}.
\end{equation}
We thus have that 
$$
\begin{aligned}
\sum_{\pi\in S_N} \kappa(\rho,\gamma_\pi) &= \sum_{X\in \UU(r,c)} n(X)\, \mathbf{k}_1(\rho,\gamma_\pi) \mathbf{k}_2(\rho,\gamma_\pi) \\
& =\sum_{X\in \UU(r,c)}  \frac{r_1!\cdots r_d! \cdot c_1!\cdots c_d!}{\prod_{ij}^d x_{ij}!} \prod_{ij}^d k_{ij}^{x_{ij}} \frac{\prod_{ij}^d x_{ij}!}{r_1!\cdots r_d! \cdot c_1!\cdots c_d!}= \,T(r,c\,;K).\end{aligned}$$
\end{proof}

{\small{\bibliographystyle{apa}
\bibliography{bib_short}}}

\begin{thebibliography}{}

\bibitem[\protect\astroncite{Ahuja et~al.}{1993}]{ahuja1993network}
Ahuja, R., Magnanti, T., and Orlin, J. (1993).
\newblock {\em Network Flows: Theory, Algorithms and Applications}.
\newblock Prentice Hall.

\bibitem[\protect\astroncite{Andoni et~al.}{2009}]{indyk2009}
Andoni, A., Ba, K.~D., Indyk, P., and Woodruff, D. (2009).
\newblock Efficient sketches for earth-mover distance, with applications.
\newblock In {\em Foundations of Computer Science (FOCS) 2009.}, pages 324
  --330.

\bibitem[\protect\astroncite{Barvinok}{2008}]{barvinok2008enumerating}
Barvinok, A. (2008).
\newblock Enumerating contingency tables via random permanents.
\newblock {\em Combinatorics, Probability and Computing}, 17(1):1--19.

\bibitem[\protect\astroncite{Bertsimas and
  Tsitsiklis}{1997}]{bertsimas1997introduction}
Bertsimas, D. and Tsitsiklis, J. (1997).
\newblock {\em Introduction to linear optimization}.
\newblock Athena Scientific.

\bibitem[\protect\astroncite{Brualdi}{2006}]{brualdi2006combinatorial}
Brualdi, R. (2006).
\newblock {\em Combinatorial matrix classes}.
\newblock Encyclopedia of Mathematics and Its Applications 108, Cambridge
  University Press.

\bibitem[\protect\astroncite{Chen et~al.}{2005}]{chen2005sequential}
Chen, Y., Diaconis, P., Holmes, S., and Liu, J. (2005).
\newblock Sequential monte carlo methods for statistical analysis of tables.
\newblock {\em Journal of the American Statistical Association},
  100(469):109--120.

\bibitem[\protect\astroncite{Cryan and Dyer}{2003}]{cryan2003polynomial}
Cryan, M. and Dyer, M. (2003).
\newblock A polynomial-time algorithm to approximately count contingency tables
  when the number of rows is constant.
\newblock {\em Journal of Computer and System Sciences}, 67(2):291--310.

\bibitem[\protect\astroncite{Cuturi}{2007}]{cuturi07permanents}
Cuturi, M. (2007).
\newblock Permanents, transportation polytopes and positive-definite kernels on
  histograms.
\newblock In {\em Proc. of the 20th Intern. Joint Conf. on Artificial
  Intelligence 2007}, pages 732 -- 737.

\bibitem[\protect\astroncite{Cuturi et~al.}{2007}]{cuturi07kernelSHORT}
Cuturi, M., Vert, J.-P., Birkenes, {\O.}., and Matsui, T. (2007).
\newblock A kernel for time series based on global alignments.
\newblock In {\em Proceedings of ICASSP}, volume~II, pages 413 -- 416.

\bibitem[\protect\astroncite{Diaconis and Efron}{1985}]{diaconisefron}
Diaconis, P. and Efron, B. (1985).
\newblock Testing for independence in a two-way table: new interpretations of
  the chi-square statistic.
\newblock {\em The Annals of Statistics}, 13(3):845--913.

\bibitem[\protect\astroncite{Everitt}{1992}]{everitt1992analysis}
Everitt, B. (1992).
\newblock {\em The analysis of contingency tables}.
\newblock Chapman \& Hall/CRC.

\bibitem[\protect\astroncite{Fulton}{1997}]{fulton1997young}
Fulton, W. (1997).
\newblock {\em Young tableaux: with applications to representation theory and
  geometry}, volume~35.
\newblock Cambridge Univ Press.

\bibitem[\protect\astroncite{Good}{1976}]{good1976}
Good, I.~J. (1976).
\newblock On the application of symmetric dirichlet distributions and their
  mixtures to contingency tables.
\newblock {\em The Annals of Statistics}, 4(6):pp. 1159--1189.

\bibitem[\protect\astroncite{Haussler}{1999}]{haussler99convolution}
Haussler, D. (1999).
\newblock Convolution kernels on discrete structures.
\newblock Technical report, UCSC.
\newblock USCS-CRL-99-10.

\bibitem[\protect\astroncite{Hoffman}{1961}]{hoffman1961simple}
Hoffman, A. (1961).
\newblock On simple linear programming problems.
\newblock In {\em Proceedings of Symposia in Pure Mathematics}, volume~7, pages
  317--327. American Mathematical Society.

\bibitem[\protect\astroncite{Jing et~al.}{2004}]{emd2004}
Jing, F., Li, M., Zhang, H.-J., and Zhang, B. (2004).
\newblock An efficient and effective region-based image retrieval framework.
\newblock {\em Image Processing, IEEE Transactions on}, 13(5):699 --709.

\bibitem[\protect\astroncite{Knuth}{1970}]{Knuth70}
Knuth, D.~E. (1970).
\newblock Permutations, matrices, and generalized {Young} tableaux.
\newblock {\em Pacific J. Math.}, 34:709--727.

\bibitem[\protect\astroncite{Lauritzen}{1982}]{lauritzen1982lectures}
Lauritzen, S. (1982).
\newblock {\em Lectures on contingency tables}.
\newblock Aalborg Univ. Press.

\bibitem[\protect\astroncite{Naor and Schechtman}{2007}]{naor-2005}
Naor, A. and Schechtman, G. (2007).
\newblock Planar earthmover is not in l$_{\mbox{1}}$.
\newblock {\em SIAM J. Comput.}, 37(3):804--826.

\bibitem[\protect\astroncite{Pele and Werman}{2009}]{Pele-iccv2009}
Pele, O. and Werman, M. (2009).
\newblock Fast and robust earth mover's distances.
\newblock In {\em ICCV'09}.

\bibitem[\protect\astroncite{Rachev and R{\"u}schendorf}{1998}]{rachev1998mass}
Rachev, S. and R{\"u}schendorf, L. (1998).
\newblock {\em Mass Transportation Problems: Theory}, volume~1.
\newblock Springer Verlag.

\bibitem[\protect\astroncite{Rubner et~al.}{1997}]{rubner1997earth}
Rubner, Y., Guibas, L., and Tomasi, C. (1997).
\newblock The earth mover’s distance, multi-dimensional scaling, and
  color-based image retrieval.
\newblock In {\em Proceedings of the ARPA Image Understanding Workshop}, pages
  661--668.

\bibitem[\protect\astroncite{Shin et~al.}{2011}]{shin2011mapping}
Shin, K., Cuturi, M., and Kuboyama, T. (2011).
\newblock Mapping kernels for trees.
\newblock {\em Proc. of ICML 2011}.

\bibitem[\protect\astroncite{Shin and Kuboyama}{2008}]{shin2008generalization}
Shin, K. and Kuboyama, T. (2008).
\newblock {A generalization of Haussler's convolution kernel: mapping kernel}.
\newblock In {\em Proceedings of the 25th international conference on Machine
  learning}, pages 944--951.

\bibitem[\protect\astroncite{Stougie}{2002}]{stougie2002polynomial}
Stougie, L. (2002).
\newblock A polynomial bound on the diameter of the transportation polytope.
\newblock Technical report.

\bibitem[\protect\astroncite{Vert et~al.}{2004}]{VerSaiAku04}
Vert, J.-P., Saigo, H., and Akutsu, T. (2004).
\newblock Local alignment kernels for protein sequences.
\newblock In Sch{\"o}lkopf, B., Tsuda, K., and Vert, J.-P., editors, {\em
  Kernel Methods in Computational Biology}. MIT Press.

\bibitem[\protect\astroncite{Villani}{2009}]{villani09}
Villani, C. (2009).
\newblock {\em Optimal transport: old and new}, volume 338.
\newblock Springer Verlag.

\bibitem[\protect\astroncite{Yates}{1934}]{yates1934contingency}
Yates, F. (1934).
\newblock Contingency tables involving small numbers and the $\chi$ 2 test.
\newblock {\em Supplement to the Journal of the Royal Statistical Society},
  1(2):217--235.

\bibitem[\protect\astroncite{Zhang et~al.}{2006}]{emd2006}
Zhang, J., Marszalek, M., Lazebnik, S., and Schmid, C. (2006).
\newblock Local features and kernels for classification of texture and object
  categories: A comprehensive study.
\newblock In {\em CVPRW '06}, page~13.

\end{thebibliography}
\end{document}